\documentclass[letterpaper, 10 pt, conference]{ieeeconf}  % Comment this line out
\IEEEoverridecommandlockouts                              % This command is only
\usepackage{graphics} % for pdf, bitmapped graphics files
\usepackage{epsfig} % for postscript graphics files
\usepackage{mathptmx} % assumes new font selection scheme installed
\usepackage{times} % assumes new font selection scheme installed
\usepackage{amsmath} % assumes amsmath package installed
\usepackage{amssymb}  % assumes amsmath package installed

\usepackage{mathtools}

\usepackage{algorithm}
\usepackage{algorithmic}
\usepackage{xcolor}
\usepackage{hyperref}

\usepackage{color}

\usepackage{blindtext}
\usepackage{enumerate}
\usepackage{graphicx}
\usepackage{amsfonts, amsmath, bm, amssymb}
\usepackage{dsfont}
\usepackage{wrapfig}
\usepackage{subcaption}

\usepackage{pifont}% http://ctan.org/pkg/pifont
\usepackage{xspace}
\makeatletter
\DeclareRobustCommand\onedot{\futurelet\@let@token\@onedot}
\def\@onedot{\ifx\@let@token.\else.\null\fi\xspace}

\makeatother

\newcommand{\Dc}{\mathcal{D}}

\newcommand{\Lc}{\mathcal{L}}
\newcommand{\Mc}{\mathcal{M}}
\newcommand{\Nc}{\mathcal{N}}

\newcommand{\Rc}{\mathcal{R}}

\newcommand{\Eb}{\mathbb{E}}

\newcommand{\Rb}{\mathbb{R}}

\newcommand{\gv}{\mathbf{g}}

\newcommand{\vv}{\mathbf{v}}
\newcommand{\wv}{\mathbf{w}}
\newcommand{\xv}{\mathbf{x}}
\newcommand{\yv}{\mathbf{y}}
\newcommand{\zv}{\mathbf{z}}

\ifx\BlackBox\undefined
\newcommand{\BlackBox}{\rule{1.5ex}{1.5ex}}  % end of proof
\fi
\ifx\QED\undefined
\def\QED{~\rule[-1pt]{5pt}{5pt}\par\medskip}
\fi
\ifx\proof\undefined
\newenvironment{proof}{\par\noindent{\em Proof:\ }}{\hfill\BlackBox\\}
\fi
\ifx\theorem\undefined
\newtheorem{theorem}{Theorem}
\fi
\ifx\example\undefined
\newtheorem{example}{Example}
\fi
\ifx\property\undefined
\newtheorem{property}{Property}
\fi
\ifx\lemma\undefined
\newtheorem{lemma}{Lemma}
\fi
\ifx\proposition\undefined
\newtheorem{proposition}{Proposition}
\fi
\ifx\fact\undefined
\newtheorem{fact}{Fact}
\fi
\ifx\remark\undefined
\newtheorem{remark}{Remark}
\fi
\ifx\corollary\undefined
\newtheorem{corollary}{Corollary}
\fi
\ifx\definition\undefined
\newtheorem{definition}{Definition}
\fi
\ifx\conjecture\undefined
\newtheorem{conjecture}{Conjecture}
\fi
\ifx\axiom\undefined
\newtheorem{axiom}[theorem]{Axiom}
\fi
\ifx\claim\undefined
\newtheorem{claim}[theorem]{Claim}
\fi
\ifx\assumption\undefined
\newtheorem{assumption}{Assumption}
\fi
\ifx\question\undefined
\newtheorem{question}{Question}
\fi
\ifx\problem\undefined
\newtheorem{problem}{Problem}
\fi

\newcommand{\clip}{\text{clip}}

\title{\LARGE \bf
Almost Sure Convergence Analysis of Differentially Private Stochastic Gradient Methods
}

\author{Amartya Mukherjee and Jun Liu% <-this % stops a space
\thanks{Amartya Mukherjee and Jun Liu are with the Department of Applied Mathematics, University of Waterloo, Waterloo, Ontario, Canada N2L 3G1 (email: {\tt\small (a29mukhe,j.liu)@uwaterloo.ca}).}%
}

\begin{document}

\maketitle

\begin{abstract}
    Differentially private stochastic gradient descent (DP-SGD) has become the standard algorithm for training machine learning models with rigorous privacy guarantees. Despite its widespread use, the theoretical understanding of its long-run behavior remains limited: existing analyses typically establish convergence in expectation or with high probability, but do not address the almost sure convergence of single trajectories. In this work, we prove that DP-SGD converges almost surely under standard smoothness assumptions, both in nonconvex and strongly convex settings, provided the step sizes satisfy some standard decaying conditions. Our analysis extends to momentum variants such as the stochastic heavy ball (DP-SHB) and Nesterov’s accelerated gradient (DP-NAG), where we show that careful energy constructions yield similar guarantees. 
    % The proofs overcome two main challenges unique to the private setting: gradient clipping, which introduces a persistent bias, and Gaussian noise, which disrupts unbiasedness. By adapting supermartingale techniques and bounding the clipping bias, we establish that the gradient norms of DP-SGD (and its momentum variants) vanish almost surely. 
    These results provide stronger theoretical foundations for differentially private optimization and suggest that, despite privacy-induced distortions, the algorithm remains pathwise stable in both convex and nonconvex regimes.
    % I want to prove that differentially private SGD converges almost surely.
\end{abstract}

\section{Introduction}

In the training of machine learning models, maintaining the privacy of training data is of paramount importance. Particularly in domains such as health and finance, it is important that an adversary cannot reconstruct training data from the trained models. Unfortunately, generative models risk overfitting to their training data, thus generating data indistinguishable from the training set and compromising the privacy of users. For example, the work of \cite{chen2023pathway} reviews risks of leaking training data in image and text generation models.

To protect the privacy of training data, a mathematically rigorous framework for privacy titled differential privacy (DP \cite{dwork2006calibrating}) has gained interest in recent years. Differentially private stochastic gradient descent (DP-SGD \cite{abadi2016deep}) is a modification of stochastic gradient descent (SGD) that offers privacy guarantees by introducing gradient clipping and noise injection. While DP guarantees strong privacy protection, it often comes at the cost of slower convergence and degraded model utility due to the bias induced by clipping and the noise injection. This paper will analyze the convergence rates of various DP-SGD methods under different noise injection schemes and dataset conditions. We explore how gradient clipping and noise scaling affect model performance.

Since the advent of DP-SGD, its convergence analysis has been of interest to both the machine learning and control communities. Recent literature has focused on the optimization and generalization trade-offs introduced by differential privacy, but not on its almost-sure stability. For example, the work of \cite{fang2023improved} provides convergence analysis on a modified DP-SGD that replaces gradient clipping with an affine function of the gradient of the objective function.
The authors of \cite{tang2024dp} extend the analysis to momentum-based variants and show that the additive Gaussian noise used for privacy can dominate the second-moment estimates in adaptive methods like Adam, effectively neutralizing their curvature adaptation and creating severe ill-conditioning under heavy-tailed data distributions. They demonstrate that bias-corrected DP-Adam (DP-AdamBC) mitigates this issue by subtracting the variance of the DP noise, improving convergence on imbalanced datasets.
The comprehensive work of \cite{koloskova2023revisiting} derives convergence rates of SGD with clipping in deterministic and stochastic settings.
Lastly, DP has also been studied in distributed optimization settings \cite{huang2024differential}, symbolic systems \cite{chen2023differential}, and multi-agent systems \cite{katewa2019differential}.

Overall, most existing analyses provide convergence guarantees only in expectation or with high probability, leaving open the question of whether individual trajectories stabilize. This gap is critical, since practical deployments of DP-SGD often train for many epochs under noisy, biased gradient updates introduced by clipping and Gaussian perturbation.
Our paper builds upon these works by providing almost sure convergence guarantees for DP-SGD in convex and non-convex settings. We approach this by proving that a weighted average of the norm of the gradient of the objective function converges almost surely, and therefore that the best iterate converges.
We extend our analysis further to variants of SGD that include momentum, where we also provide almost sure convergence guarantees of the last iterates.

\section{Preliminaries and Assumptions}

We provide a formal definitions of DP and SGD, and introduce some assumptions commonly used in the convergence analysis of SGD \cite{liu2022almost,doan2022finite}.

\begin{definition}[Differential Privacy (DP) \cite{dwork2006calibrating}]
    A randomized mechanism $\Mc:\Dc\to\Rc$ with domain $\Dc$ and range $\Rc$ satisfies $(\epsilon,\delta)$ differential privacy, where $\epsilon>0,\delta>0$, if for any datasets $d,d'\subset\Dc$ differing by at most one entry, and for any subset of outputs $S\subset\Rc$, it holds that
    \begin{equation}
        P(\Mc(d)\in S)\leq e^{\epsilon}P(\Mc(d')\in S)+\delta.
    \end{equation}
\end{definition}

\textbf{Problem statement:} We are interested in solving the following unconstrained minimization problem
\begin{equation}
    \min_{\xv\in\Rb^d}f(\xv),
\end{equation}
where $f:\Rb^d\to\Rb$, using stochastic gradient methods that satisfy $(\epsilon,\delta)$-DP.
Let $f^*$ be the true minimum. In convex settings, we want to prove that $f(\xv_t)-f^*\to 0$ as $t\to\infty$. In non-convex settings, we want to prove that $\nabla f(\xv_t)\to 0$ as $t\to\infty$.

\begin{definition}[Stochastic Gradient Descent (SGD)]
    The iteration of SGD is given by
    \begin{equation}
        \xv_{t+1}=\xv_t-\alpha_t \gv_t,
    \end{equation}
    where $\gv_t=\nabla f(\xv;\xi_t)$ is the stochastic gradient at $\xv_t$ with a random process $\xi_t$ and $\alpha_t$ is a step size. Throughout this paper, we denote $\nabla f(\xv_t):=\Eb[\gv_t]$ as the expectation of the stochastic gradient over all $\xi_t$.
\end{definition}

\begin{definition}[Differentially Private SGD (DP-SGD) \cite{abadi2016deep}]\label{def:DP_SGD}
    DP-SGD is a modification of SGD, where gradients are clipped and noise is added to the clipped gradients.
    \begin{equation}\label{eq:DP_SGD}
        \xv_{t+1}=\xv_t-\alpha_t \gv_t^{DP},
    \end{equation}
    where the differentially private stochastic gradient $\gv_t^{DP}$ is given by
    \begin{equation}
        \gv_t^{DP}=\clip_q(\nabla f(\xv_t;\xi_t))+q\zeta_t,
    \end{equation}
    where the $\clip_q$ function is defined for $q>0$ by
    \begin{equation}
        \clip_q(\nabla f(\xv_t;\xi_t))=\min\left(1,\frac{q}{\|\nabla f(\xv_t;\xi_t)\|}\right)\nabla f(\xv_t;\xi_t),
    \end{equation}
    where $\zeta_t\sim\Nc(0,\sigma_{DP}^2I)$, and $\|\cdot\|$ denotes the 2-norm.
\end{definition}

\begin{remark}
    DP-SGD can be adapted to other stochastic gradient methods such as stochastic heavy ball and stochastic Nesterov accelerated gradient.
    A DP-SGD update satisfies $(\epsilon,\delta)$-DP if $\sigma_{DP}^2$ and $q$ satisfy $\sigma_{DP}^2>2\log(1.25/\delta)q^2/\epsilon^2$.
\end{remark}

Based on the formulation of DP-SGD, we also introduce the notion of clipping probability.

\begin{definition}[Clipping Probability]\label{def:clipping}
    Define the process $\eta_t$ as the clipping probability
    \begin{equation}
        \eta_t=P(\|\nabla f(\xv_t,\xi_t)\|>q|\xv_t).
    \end{equation}
\end{definition}

We make the following assumptions that are commonly used in the SGD literature \cite{nesterov2004introductory}.

\begin{assumption}[$L$-smoothness]\label{as:L_smoothness}
    $f$ is bounded from below by $f^*:=\inf_{\xv\in\Rb^d}f(\xv)$ and its gradient $\nabla f$ is $L$-Lipschitz i.e. $\|\nabla f(\xv)-\nabla f(\yv)\|\leq L\|\xv-\yv\|$, for all $\xv,\yv\in\Rb^d$.
\end{assumption}

\begin{assumption}[$\mu$-strongly convex]\label{as:mu_strongly_convex}
    There exists a positive constant $\mu>0$ such that
    \begin{equation}
        f(\yv)\geq f(\xv)+\langle\nabla f(\xv),\yv-\xv\rangle+\frac{\mu}{2}\|\yv-\xv\|^2,
    \end{equation}
    for all $\xv,\yv\in\Rb^d$
    A consequence of $f$ being $\mu$-strongly convex is that
    \begin{equation}
        \frac{1}{2\mu}\|\nabla f(\xv)\|^2\geq f(\xv)-f^*.
    \end{equation}
\end{assumption}

% \begin{assumption}[ABC condition]\label{as:ABC}
%     There exist constants $A,B,C\geq 0$ such that
%     \begin{equation}
%         \Eb\|\nabla f(\xv;\xi)\|^2\leq A(f(\xv)-f^*)+B\|\nabla f(\xv)\|^2+C
%     \end{equation}
% \end{assumption}

% \begin{remark}
%     This assumption has been proposed in \cite{khaled2020better} as "the weakest assumption" for SGD.
% \end{remark}

\begin{assumption}[Directional Invariance]\label{as:direction}
    There exists a constant $D>0$ such that
    \begin{equation}
        \Eb\left[\langle\nabla f(\xv_t),\frac{\nabla f(\xv_t;\xi)}{\|\nabla f(\xv_t;\xi)\|}\rangle\mid\xv_t\right]\geq D\|\nabla f(\xv_t)\|
    \end{equation}
    holds whenever $\nabla f(\xv_t)\neq 0$.
\end{assumption}

\begin{remark}
    This assumption essentially states that the direction of $\nabla f(\xv;\xi)$ is preserved if we normalize it, and that the distribution of $\nabla f(\xv;\xi)$ is not extremely skewed.
\end{remark}

\section{Background and Lemmas on Supermartingales}

The analysis in this paper follows from the following result derived in \cite{robbins1971convergence}. From this section onward, we use the shorthand notation $\Eb_t[\cdot]:=\Eb[\cdot|\xv_t]$.

\begin{proposition}\label{prop:1}
    Let $\{X_t\}$, $\{Y_t\}$, and $\{Z_t\}$ be three sequences of random variables that are adapted to a filtration $\{\mathcal{F}_t\}$. Let $\{\gamma_t\}$ be a sequence of nonnegative real numbers such that $\prod_{t=1}^{\infty} (1 + \gamma_t) < \infty$. Suppose that the following conditions hold:
    \begin{enumerate}
        \item $X_t, Y_t, Z_t$ are nonnegative for all $t \geq 1$.
        \item $\mathbb{E}[Y_{t+1} | \mathcal{F}_t] \leq (1 + \gamma_t) Y_t - X_t + Z_t$ for all $t \geq 1$.
        \item $\sum_{t=1}^{\infty} Z_t < \infty$ holds almost surely.
    \end{enumerate}
    Then, we have
    \begin{equation}
        \sum_{t=1}^{\infty} X_t < \infty \quad \text{almost surely,}
    \end{equation}
    and $Y_t$ converges almost surely.
\end{proposition}

% To establish a concrete converge rate, we refer to the lemma from \cite{liu2022almost}.

% \begin{lemma}[Lemma 1 of \cite{liu2022almost}]\label{lem:1}
%     If $\{Y_t\}$ is a sequence of nonnegative random variables satisfying
%     \begin{equation}
%         \Eb[Y_{t+1}|\Fc_t]\leq (1-c_1\alpha_t)Y_t+c_2\alpha_t^2,
%     \end{equation}
%     for all $t\geq 1$ and $\alpha_t=\Theta(\frac{1}{t^{1-\theta}})$ for some $\theta\in(0,\frac{1}{2})$, and $c_1,c_2$ are positive constants. Then, for any $\gamma\in(2\theta,1)$,
%     \begin{equation}
%         Y_t=o(\frac{1}{t^{1-\gamma}}),
%     \end{equation}
%     almost surely.
% \end{lemma}

The following result from \cite{liu2022almost} is used for convergence results in non-convex settings.

\begin{lemma}[Lemma 2 of \cite{liu2022almost}]\label{lem:jy}
Let $\{X_t\}$ be a sequence of nonnegative real numbers and $\{\alpha_t\}$ be a decreasing sequence of positive real numbers such that the following conditions hold:
\begin{equation}
    \sum_{t=1}^{\infty} \alpha_t X_t < \infty, \quad \sum_{t=1}^{\infty} \frac{\alpha_t}{\sum_{i=1}^{t-1} \alpha_i} = \infty.
\end{equation}
% Define 
% \begin{equation}
% w_t = \frac{2\alpha_t}{\sum_{i=1}^{t} \alpha_i}., \quad Y_1 = X_1,
% \end{equation}
% and recursively set
% \begin{equation}
% Y_{t+1} = (1 - w_t) Y_t + w_t X_t, \quad \text{for } t \geq 1.
% \end{equation}
Then,
% \begin{equation}
% Y_t = o \left( \frac{1}{\sum_{i=1}^{t-1} \alpha_i} \right),
% \end{equation}
% and
\begin{equation}
\min_{1 \leq i \leq t} X_i = o \left( \frac{1}{\sum_{i=1}^{t-1} \alpha_i} \right).
\end{equation}
\end{lemma}

We derive some properties of the differentially private stochastic gradient that aids in our core theory.

\begin{proposition}\label{prop:dp_expectations}
    For all $\xv_t$ and $\xi_t$,
    \begin{align}
        \Eb_t\|\gv_t^{DP}\|^2&\leq q^2+q^2d\sigma_{DP}^2.
    \end{align}
    Furthermore, if Assumption \ref{as:direction} holds, then
    \begin{align}
        -\Eb_t[\langle\nabla f(\xv_t),\gv_t^{DP}\rangle]
        \leq& -(1-\eta_t)\|\nabla f(\xv_t)\|^2-D\eta_tq\|\nabla f(\xv_t)\|.
    \end{align}
\end{proposition}

\begin{proof}
    We first expand $\Eb_t\|\gv_t^{DP}\|^2$:
    \begin{align*}
        \Eb_t\|\gv_t^{DP}\|^2
        &=\Eb_t\langle\clip_q(\nabla f(\xv_t;\xi_t))+q\zeta_t,\clip_q(\nabla f(\xv_t;\xi_t))+q\zeta_t\rangle.
    \end{align*}
    Since $\xi_t$ and $\zeta_t$ are independent, we can separate the terms:
    \begin{align*}
        \Eb_t\|\gv_t^{DP}\|^2
        =&~\Eb_t\|\clip_q(\nabla f(\xv_t;\xi_t))\|^2+\Eb_t\|q\zeta_t\|^2\\&+2\Eb_t\langle\clip_q(\nabla f(\xv_t;\xi_t)),q\zeta_t\rangle\\
        =&~\Eb_t\|\clip_q(\nabla f(\xv_t;\xi_t))\|^2+q^2d\sigma_{DP}^2+0\\
        \leq&~q^2+q^2d\sigma_{DP}^2,
    \end{align*}
    where we exploit that $\|\clip_q(\cdot)\|\leq q$ for any vector.
    If Assumption \ref{as:direction} holds, then
    \begin{align*}
        -\Eb_t[\langle\nabla f(\xv_t),\gv_t^{DP}\rangle]
        =&-(1-\eta_t)\|\nabla f(\xv_t)\|^2\nonumber\\&-\eta_tq\Eb_t[\langle\nabla f(\xv_t),\tfrac{\nabla f(\xv_t;\xi_t)}{\|\nabla f(\xv;\xi_t)\|}\rangle]\\
        \leq& -(1-\eta_t)\|\nabla f(\xv_t)\|^2-D\eta_tq\|\nabla f(\xv_t)\|.
    \end{align*}
\end{proof}

\section{Almost Sure Convergence Rate Analysis with Privacy Guarantees}

In this section, we derive convergence rates for three differentially private stochastic gradient methods. we use the following notation:
\begin{align}
    &\Phi_t(\xv):=(1-\eta_t)\|\nabla f(\xv)\|^2+D\eta_tq\|\nabla f(\xv)\|,\\
    &\Phi_t^\mu(\xv):=(1-\eta_t)(2\mu(f(\xv)-f^*))+D\eta_tq\sqrt{2\mu(f(\xv)-f^*)}.
\end{align}
If Assumption \ref{as:mu_strongly_convex} holds, then $\Phi_t^\mu(\xv)\leq \Phi_t(\xv)$ as $2\mu(f(\xv)-f^*)\leq\|\nabla f(\xv)\|^2$.

\subsection{Stochastic Gradient Descent}

The iteration of DP-SGD is given in Definition \ref{def:DP_SGD}.

\begin{theorem}[Convergence of DP-SGD]
    Consider the iterates of DP-SGD. Let $\eta_t$ be the clipping probability defined in Definition \ref{def:clipping}.
    If Assumptions \ref{as:L_smoothness}, \ref{as:direction}  hold and $\alpha_t=\Theta(\frac{1}{t^{1-\theta}})$ for some $\theta\in(0,\frac{1}{2})$, then
    \begin{equation}
        \min_{1\leq i\leq t}\Phi_i(\xv_i)=o\Big(\big(\sum_{i=1}^{t-1}\alpha_i\big)^{-1}\Big)
    \end{equation}
    almost surely. Furthermore, if Assumption \ref{as:mu_strongly_convex} holds, then
    \begin{align}
        \min_{1\leq i\leq t}\Phi_i^\mu(\xv_i)=o\Big(\big(\sum_{i=1}^{t-1}\alpha_i\big)^{-1}\Big)
    \end{align}
\end{theorem}

\begin{proof}
    By Assumption \ref{as:L_smoothness}, we have
    \begin{equation*}
        f(\xv_{t+1})\leq f(\xv_{t})-\alpha_t\langle\nabla f(\xv_t),\gv_t^{DP}\rangle+\frac{L\alpha_t^2}{2}\|\gv_t^{DP}\|^2.
    \end{equation*}
    Taking the expectation $\Eb_t[\cdot]$ of both sides gives
    \begin{equation*}\label{eq:pf_sgd_1}
    \begin{split}
        \Eb_t[f(\xv_{t+1})-f^*]\leq&~ f(\xv_t)-f^*-\alpha_t\Eb_t[\langle\nabla f(\xv_t),\gv_t^{DP}\rangle]\\&+\frac{L\alpha_t^2}{2}\Eb_t\|\gv_t^{DP}\|^2.
    \end{split}
    \end{equation*}
    By Proposition \ref{prop:dp_expectations},
    \begin{align}\label{eq:g_t^2}
        &\Eb_t\|\gv_t^{DP}\|^2\leq q^2+q^2d\sigma_{DP}^2,\\
        -&\Eb_t[\langle\nabla f(\xv_t),\gv_t^{DP}\rangle]
        \leq -\Phi_t(\xv_t).
    \end{align}
    Plugging these terms simplifies the expression to
    \begin{align}\label{eq:pf_sgd_2}
        \Eb_t[f(\xv_{t+1})-f^*]
        \leq &f(\xv_t)-f^*-\alpha_t\Phi_t(\xv_t)\nonumber\\&+\frac{L\alpha_t^2}{2}q^2(1+d\sigma_{DP}^2).
    \end{align}
    \textbf{Non-Convex Case.} By Proposition \ref{prop:1}, equation \eqref{eq:pf_sgd_2} gives
    \begin{align*}
        \sum_{t=1}^{\infty}\alpha_t\Phi_t(\xv_t)&<\infty.
    \end{align*}
    Thus from Lemma \ref{lem:jy},
    \begin{equation*}
        \min_{1\leq i\leq t}\Phi_i(\xv_i)=o\Big(\big(\sum_{i=1}^{t-1}\alpha_i\big)^{-1}\Big).
    \end{equation*}
    \textbf{Strongly Convex Case.} From strong convexity:
    \begin{equation*}
    \|\nabla f(\xv_t)\|^2 \geq 2\mu(f(\xv_t) - f^*).
    \end{equation*}
    equation \eqref{eq:pf_sgd_2} simplifies to
    \begin{align*}
        \Eb_t[f(\xv_{t+1})-f^*]\leq& (1-2\alpha_t(1-\eta_t)\mu)(f(\xv_t)-f^*)\\&-\alpha_tD\eta_tq\|\nabla f(\xv_t)\|+\frac{L\alpha_t^2}{2}q^2(1+d\sigma_{DP}^2)\\
        \leq& (1-2\alpha_t(1-\eta_t)\mu)(f(\xv_t)-f^*)\\&-\alpha_tD\eta_tq\sqrt{2\mu(f(\xv_t) - f^*)}\\&+\frac{L\alpha_t^2}{2}q^2(1+d\sigma_{DP}^2).
    \end{align*}
    By Proposition \ref{prop:1}, we conclude that
    \begin{equation*}
        \sum_{t=1}^{\infty}\alpha_t[(1-\eta_t)(2\mu(f(\xv_t)-f^*))+D\eta_tq\sqrt{2\mu(f(\xv_t)-f^*)}]<\infty,
    \end{equation*}
    and therefore from Lemma \ref{lem:jy},
    \begin{equation*}
    \begin{split}
        \min_{1\leq i\leq t}\Phi_i^\mu(\xv_i)=o\Big(\big(\sum_{i=1}^{t-1}\alpha_i\big)^{-1}\Big),
    \end{split}
    \end{equation*}
    which concludes the proof.
    % Almost sure convergence follows from Lemma \ref{lem:1}.
\end{proof}

\subsection{Stochastic Heavy-Ball Method}

The iteration of the differentially private stochastic heavy-ball (DP-SHB) method is given by
\begin{equation}
    \xv_{t+1}=\xv_t-\alpha_t\gv_t^{DP}+\beta(\xv_t-\xv_{t-1}),
\end{equation}
where $\beta\in[0,1)$ is the weight given to the momentum component. To simplify our analysis, we express the DP-SHB iteration as a system of two variables. Define
\begin{equation}\label{eq:parameterization}
    \zv_t=\xv_t+\frac{\beta}{1-\beta}\vv_t,\quad\vv_t=\xv_t-\xv_{t-1}.
\end{equation}
The iteration of SHB can be rewritten as
\begin{align}
    \vv_{t+1}&=\beta\vv_t-\alpha \gv_t^{DP},\quad
    \zv_{t+1}=\zv_t-\frac{\alpha_t}{1-\beta}\gv_t^{DP}.
\end{align}
This update rule is derived in \cite{liu2024almost}.

\begin{theorem}[Convergence of DP-SHB]\label{thm:SHB}
    Let $\{\xv_t\}$ be the iterates of DP-SHB. Let $\eta_t$ be the clipping probability defined in Definition~\ref{def:clipping}. Define the differentially private stochastic gradient as in equation \eqref{eq:DP_SGD}.
    If Assumptions \ref{as:L_smoothness}, \ref{as:direction}  hold and $\alpha_t=\Theta(\frac{1}{t^{1-\theta}})$ for some $\theta\in(0,\frac{1}{2})$, then
    \begin{equation}
        \min_{1\leq i\leq t}\Phi_i(\xv_i)=o\Big(\big(\sum_{i=1}^{t-1}\alpha_i\big)^{-1}\Big)
    \end{equation}
    almost surely. Furthermore, if Assumption \ref{as:mu_strongly_convex} holds, then
    \begin{align}
        \min_{1\leq i\leq t}\Phi_i^\mu(\xv_i)=o\Big(\big(\sum_{i=1}^{t-1}\alpha_i\big)^{-1}\Big)
    \end{align}
\end{theorem}

\begin{proof} Define the energy function
\[
Y_t := f(\zv_t)-f^* + c\|\vv_t\|^2.
\]
% (Any fixed $c\ge L\beta/(1-\beta)$ suffices; this concrete choice simplifies signs below.)
For a constant $c>0$. Moreover,
\[
\|\vv_{t+1}\|^2
=\beta^2\|\vv_t\|^2 + \alpha_t^2\|\gv^{\mathrm{DP}}_t\|^2 - 2\alpha_t\beta\langle \vv_t,\gv^{\mathrm{DP}}_t\rangle.
\]
Taking $\Eb_t[\cdot]$ of both sides gives
\[
\Eb_t\|\vv_{t+1}\|^2
\leq\beta^2\|\vv_t\|^2 + \alpha_t^2\Eb_t\|\gv^{\mathrm{DP}}_t\|^2 - 2\alpha_t\beta\Eb_t\langle \vv_t,\gv^{\mathrm{DP}}_t\rangle.
\]
Using Young's inequality and Proposition \ref{prop:dp_expectations}, we introduce a constant $c_1>0$ such that
\begin{align}\label{eq:v_t}
\Eb_t\|\vv_{t+1}\|^2
\leq&~\beta^2\|\vv_t\|^2+\alpha_t^2q^2(1+d\sigma_{DP}^2)\\&+c_1\|\vv_t\|^2+\tfrac{\alpha_t^2\beta^2}{c_1}\Eb_t\|\gv_t^{DP}\|^2\nonumber\\
\leq&~(\beta^2+c_1)\|\vv_t\|^2+\alpha_t^2q^2(1+\tfrac{\beta^2}{c_1})(1+d\sigma_{DP}^2).
\end{align}
% \paragraph{One-step descent.}
By Assumption \ref{as:L_smoothness},
\begin{equation*}
    f(\zv_{t+1})\leq f(\zv_t)-\tfrac{\alpha_t}{1-\beta}\langle\nabla f(\zv_t),\gv_t^{DP}\rangle+\tfrac{L\alpha_t^2}{2(1-\beta)^2}\|\gv_t^{DP}\|^2.
\end{equation*}
Taking $\Eb_t[\cdot]$ of both sides gives
\begin{align}\label{eq:SHB_1}
    \Eb_tf(\zv_{t+1})
    \leq&~ f(\zv_t)-\tfrac{\alpha_t}{1-\beta}\Eb_t\langle\nabla f(\zv_t),\gv_t^{DP}\rangle+\tfrac{L\alpha_t^2}{2(1-\beta)^2}\Eb_t\|\gv_t^{DP}\|^2.
    % \leq&~ f(\zv_t)-\tfrac{\alpha_t}{1-\beta}\Eb_t\langle\nabla f(\zv_t),\gv_t^{DP}\rangle+\tfrac{L\alpha_t^2}{2(1-\beta)^2}q^2(1+d\sigma_{DP}^2)
\end{align}
To bound $\Eb_t\langle\nabla f(\zv_t),\gv_t^{DP}\rangle$, we can expand it:
\begin{align*}
    &-\Eb_t\langle\nabla f(\zv_t),\gv_t^{DP}\rangle\\
    &=-\Eb_t\langle\nabla f(\xv_t),\gv_t^{DP}\rangle-\Eb_t\langle\nabla f(\zv_t)-\nabla f(\xv_t),\gv_t^{DP}\rangle\\
    &\leq-\Phi_t(\xv_t)+\|\nabla f(\zv_t)-\nabla f(\xv_t)\|\|\Eb_t\gv_t^{DP}\|\\
    &\leq-\Phi_t(\xv_t)+Lq\sqrt{1+d\sigma_{DP}^2}\|\zv_t-\xv_t\|\\
    &\leq-\Phi_t(\xv_t)+\tfrac{Lq\beta}{1-\beta}\sqrt{1+d\sigma_{DP}^2}\|\vv_t\|.
\end{align*}
Letting $K:=\tfrac{Lq\beta}{1-\beta}\sqrt{1+d\sigma_{DP}^2}$, equation \eqref{eq:SHB_1} simplifies to
\begin{equation*}
\Eb_tf(\zv_{t+1})\leq f(\zv_t)-\tfrac{\alpha_t}{1-\beta}\Phi_t(\xv_t)+\tfrac{\alpha_t}{1-\beta}K\|\vv_t\|+\tfrac{K^2}{2L}\alpha_t^2.
\end{equation*}
Finally, we control the extra $\|\vv_t\|$ term using Young's inequality by introducing a new constant $c_2>0$ such that
\begin{equation*}
    \tfrac{\alpha_t}{1-\beta}K\|\vv_t\|\leq\tfrac{c_2(1-\beta)}{2}\|\vv_t\|^2+\tfrac{K^2}{2c_2(1-\beta)^3}\alpha_t^2.
\end{equation*}
Combining with equation \eqref{eq:v_t}, we can now express the energy function $Y_t$ as a supermartingale
\begin{align}
    \Eb_t[&f(\zv_{t+1})-f^*+c\|\vv_{t+1}\|^2]\nonumber\\
    \leq &~f(\zv_t)-f^*-\tfrac{\alpha_t}{1-\beta}\Phi_t(\xv_t)+(\tfrac{c_2(1-\beta)}{2}+c\beta^2+cc_1)\|\vv_t\|^2\nonumber\\
    &+\alpha_t^2[\tfrac{K^2}{2L}+q^2c(1+\tfrac{\beta^2}{c_1})(1+d\sigma_{DP}^2)].
\end{align}
We can choose positive values for $c,c_1,$ and $c_2$ carefully such that $c=\tfrac{c_2(1-\beta)}{2}+c\beta^2+cc_1$, for example, $c_1=\tfrac{1-\beta^2}{2}$, $c=\tfrac{c_2}{1+\beta}$.
And let $c_3$ be the coefficient of $\alpha_t^2$ for simplification. This yields the clean recursion
\begin{equation}\label{eq:SHB_sum}
    \Eb[Y_{t+1}]\leq Y_t-\tfrac{\alpha_t}{1-\beta}\Phi_t(\xv_t)+c_3\alpha_t^2.
\end{equation}

\textbf{Non-Convex Case.} Summing the recursion and applying Proposition \ref{prop:1} give
\[
\sum_{t=1}^\infty \frac{\alpha_t}{1-\beta}
\Phi_t(\xv_t)<\infty,
\]
which by Lemma \ref{lem:jy} implies
\begin{equation}
\min_{1\le i\le t}\Phi_i(\xv_i)
= o\Big(\big(\sum_{i=1}^{t-1}\alpha_i\big)^{-1}\Big),
\end{equation}
almost surely.

\textbf{Strongly convex case.}
If $f$ is $\mu$-strongly convex, then $\|\nabla f(\xv_t)\|^2\geq 2\mu(f(\xv_t)-f^*)$. Plugging this yields
\begin{align}
    \min_{1\leq i\leq t}\Phi_i^\mu(\xv_i)=o\Big(\big(\sum_{i=1}^{t-1}\alpha_i\big)^{-1}\Big)
\end{align}
This completes the proof.
\end{proof}

\subsection{Stochastic Nesterov's Accelerated Gradient}

The iteration of the differentially private stochastic Nesterov's accelerated gradient (DP-NAG) is given by
\begin{align}
    \yv_{t+1}&=\xv_t-\alpha_t\gv_t^{DP},\\
    \xv_{t+1}&=\yv_{t}+\beta(\xv_t-\xv_{t-1}),
\end{align}
where $\beta\in[0,1)$ is the weight given to the momentum component.

\begin{theorem}[Convergence of DP-NAG]
    Let $\{\xv_t\}$ be the iterates of DP-NAG. Let $\eta_t$ be the clipping probability defined in Definition~\ref{def:clipping}. Define the differentially private stochastic gradient as in equation \eqref{eq:DP_SGD}.
    If Assumptions \ref{as:L_smoothness}, \ref{as:direction}  hold and $\alpha_t=\Theta(\frac{1}{t^{1-\theta}})$ for some $\theta\in(0,\frac{1}{2})$, then
    \begin{equation}
        \min_{1\leq i\leq t}\Phi_i(\xv_i)=o\Big(\big(\sum_{i=1}^{t-1}\alpha_i\big)^{-1}\Big)
    \end{equation}
    almost surely. Furthermore, if Assumption \ref{as:mu_strongly_convex} holds, then
    \begin{align}
        \min_{1\leq i\leq t}\Phi_i^\mu(\xv_i)=o\Big(\big(\sum_{i=1}^{t-1}\alpha_i\big)^{-1}\Big).
    \end{align}
\end{theorem}

\begin{proof}
    Define $\vv_t$ and $\zv_t$ as in equation \eqref{eq:parameterization}. The iteration of DP-NAG can be rewritten as
    \begin{align}
        \vv_{t+1}&=\beta\vv_t-\beta\alpha_t\gv_t^{DP},\quad\zv_{t+1}=\zv_t-\tfrac{\alpha_t}{1-\beta}\gv_t^{DP}.
    \end{align}
    The proof is identical to that of Theorem \ref{thm:SHB} with
    \[
    \|\vv_{t+1}\|^2
    =\beta^2[\|\vv_t\|^2 + \alpha_t^2\|\gv^{\mathrm{DP}}_t\|^2 - 2\alpha_t\langle \vv_t,\gv^{\mathrm{DP}}_t\rangle].
    \]
\end{proof}

% \vspace{-25pt}

\section{Last-iterate Convergence Analysis}

The convergence analysis results above show that the ``best'' iterate converges to zero almost surely in the strongly convex and non-convex case. To extend the almost sure convergence guarantee from best-iterate results to the last iterate, we need a device that controls oscillations of the gradient norm across iterations. Even if $\sum_t\alpha_t\|\nabla f(\xv_t)\|^2\to 0$, this condition alone does not imply that $\nabla f(\xv_t)\to 0$ because the sequence could fluctuate indefinitely. The key tool to overcome this difficulty is a lemma of \cite{orabona2020almost}, which ensures that if the weighted sum of squared gradients is finite and the gradient sequence does not vary too quickly, then the gradients themselves converge to zero. We restate a suitable version below.

\begin{lemma}[Lemma 1 of \cite{orabona2020almost}]\label{lem:3}
    Let $\{b_t\}$ and $\{\alpha_t\}$ be two nonnegative sequences and $\{w_t\}$ be a sequence of vectors. Assume $\sum_{t=1}^{\infty}\alpha_tb_t^p<\infty$ and $\sum_{t=1}^{\infty}\alpha_t=\infty$, where $p\geq 1$. Furthermore, assume that there exists some $L>0$ such that \[|b_{t+\tau}-b_t|\leq L(\sum_{i=t}^{t+\tau-1}\alpha_ib_i+\|\sum_{i=t}^{t+\tau-1}\alpha_iw_i\|),\] where $w_t$ is such that $\sum_{t=1}^{\infty}\alpha_tw_t$ converges. Then $b_t$ converges to $0$. See also Lemma 10 of \cite{liu2024almost} for the case $p>0$.
\end{lemma}

Lemma \ref{lem:3} provides a general criterion for establishing last-iterate convergence: it reduces the problem to showing that the cumulative bias and noise terms introduced by DP-SGD induced by clipping and Gaussian perturbations form a convergent series. We verify this condition under our assumptions by combining the supermartingale recursion established in Theorem \ref{thm:SHB} with additional bias control for the clipped gradient. This allows us to use Lemma \ref{lem:3} and conclude that the last-iterate gradients vanish almost surely. The formal statement is given in Theorem \ref{thm:last}.

% \begin{theorem}
%     Consider the iterates of DP-SGD, and let Assumptions \ref{as:L_smoothness} and \ref{as:direction} hold. Let the step size $\{\alpha_t\}$ satisfy $\sum_{t=1}^{\infty}\alpha_t=\infty,\sum_{t=1}^{\infty}\alpha_t^2<\infty$. Then we have $\nabla f(\xv_t)\to 0$ almost surely as $t\to\infty$.
% \end{theorem}

\begin{theorem}\label{thm:last}
    Consider the iterates of DP-SHB and DP-NAG. Let Assumptions \ref{as:L_smoothness} and \ref{as:direction} hold, and assume $q\geq 1$. Let the step size $\{\alpha_t\}$ satisfy $\sum_{t=1}^{\infty}\alpha_t=\infty,\sum_{t=1}^{\infty}\alpha_t^2<\infty$. Then we have $\nabla f(\xv_t)\to 0$ almost surely as $t\to\infty$.
\end{theorem}

% \begin{proof}
%     Consider the iterates of DP-SHB.
% \end{proof}

\begin{proof}
We revisit the convergence proof for DP-SHB. By $L$-smoothness of $f$,
\begin{align*}
\Eb_t\big[f(\zv_{t+1})\big]
\le&~
f(\zv_t) - \tfrac{\alpha_t}{1-\beta}\Eb_t\big[\langle \nabla f(\zv_t),\gv^{\mathrm{DP}}_t\rangle\big]\\
&+ \tfrac{L}{2}\Big(\tfrac{\alpha_t}{1-\beta}\Big)^2 \Eb_t\big[\|\gv^{\mathrm{DP}}_t\|^2\big]\\
\le&~
f(\zv_t) - \tfrac{\alpha_t}{1-\beta}\|\nabla f(\zv_t)\|^2\\
&- \tfrac{\alpha_t}{1-\beta}\Eb_t\big[\langle \nabla f(\zv_t),\gv^{\mathrm{DP}}_t-\nabla f(\zv_t)\rangle\big]\\
&+ \tfrac{L}{2}\Big(\tfrac{\alpha_t}{1-\beta}\Big)^2q^2(1+d\sigma_{DP}^2).
\end{align*}
Using the Cauchy-Schwarz inequality, 
\begin{equation}\label{eq:cs_v}
\begin{split}
-&\Eb_t\big[\langle \nabla f(\zv_t),\gv^{\mathrm{DP}}_t-\nabla f(\zv_t)\rangle\big]\\
\leq&~\|\nabla f(\zv_t)\|\|\Eb_t\gv^{\mathrm{DP}}_t-\nabla f(\zv_t)\|\\
\leq&~\|\nabla f(\zv_t)\|[\|\Eb_t\gv^{\mathrm{DP}}_t-\clip_q(\nabla f(\zv_t))\|\\&+\|\clip_q(\nabla f(\zv_t))-\nabla f(\zv_t)\|]\\
\leq&~\|\nabla f(\zv_t)\|[q\|\xv_t-\zv_t\|+\max(\|\nabla f(\zv_t)\|-q,0)]\\
\leq&~\|\nabla f(\zv_t)\|[\tfrac{qL\beta}{1-\beta}\|\vv_t\|+\max(\|\nabla f(\zv_t)\|-q,0)],
% &\leq~\|\nabla f(\zv_t)\|\|\nabla f(\xv_t)-\nabla f(\zv_t)\|\\
% &=~\|\nabla f(\zv_t)\|\tfrac{L\beta}{1-\beta}\|\vv_t\|.
\end{split}
\end{equation}
where $\|\Eb_t\gv^{\mathrm{DP}}_t-\clip_q(\nabla f(\zv_t))\|\leq q\|\xv_t-\zv_t\|$ comes from exploiting the $q$-Lipschitz property of $\clip_q(\cdot)$.
% where we exploit the fact that $\|\clip_q(\nabla f(\xv_t))-\nabla f(\zv_t)\|\leq \|\nabla f(\xv_t)-\nabla f(\zv_t)\|$.
Combining with equation \eqref{eq:v_t}, we have
\begin{align*}
&\Eb_t\big[f(\zv_{t+1})-f^*+\|\vv_{t+1}\|^2\big]\\
\le&~
f(\zv_t)-f^* - \tfrac{\alpha_t}{1-\beta}\min(\|\nabla f(\zv_t)\|^2,q\|\nabla f(\zv_t)\|)\\
&+\tfrac{\alpha_t}{1-\beta}\|\nabla f(\zv_t)\|\tfrac{L\beta}{1-\beta}\|\vv_t\|
+ \tfrac{L}{2}\Big(\tfrac{\alpha_t}{1-\beta}\Big)^2q^2(1+d\sigma_{DP}^2)\\
&+(\beta^2+c_1)\|\vv_t\|^2+\alpha_t^2q^2(1+\tfrac{\beta^2}{c_1})(1+d\sigma_{DP}^2)\\
\le&~
f(\zv_t)-f^* - \tfrac{\alpha_t}{1-\beta}\min(\|\nabla f(\zv_t)\|^2,q\|\nabla f(\zv_t)\|)\\
&+\tfrac{\alpha_t^2L^2\beta^2}{c_4(1-\beta)^4}\|\nabla f(\zv_t)\|^2
+(\beta^2+c_1+c_4)\|\vv_t\|^2+ \alpha_t^2C_2,
\end{align*}
where $c_1>0$ comes from using Young's inequality, and $C_2$ is the coefficient of $\alpha_t^2$.
% \[C_2:=\left[\tfrac{L}{2}\Big(\tfrac{1}{1-\beta}\Big)^2+1+\tfrac{1}{c_1}\right]q^2(1+d\sigma_{DP}^2).\]
Finally, for sufficiently large $t$, there exists a positive constant $c_5$ such that \[-\tfrac{\alpha_t}{1-\beta}+\tfrac{\alpha_t^2L^2\beta^2}{c_4(1-\beta)^4}\leq -\tfrac{c_5}{1-\beta}\alpha_t.\] This simplifies our bound to
\begin{align*}
&\Eb_t\big[f(\zv_{t+1})-f^*+\|\vv_{t+1}\|^2\big]\\
\le&~
f(\zv_t)-f^* - \tfrac{\alpha_t}{1-\beta}\min((1+c_5)\|\nabla f(\zv_t)\|^2,q\|\nabla f(\zv_t)\|)\\
&+(\beta^2+c_1+c_4)\|\vv_t\|^2+ \alpha_t^2C_2.
\end{align*}
By Proposition \ref{prop:1}, we conclude that
\begin{equation}\label{eq:conv_last_iterate1}
    \sum_{t=1}^{\infty}\alpha_t\min((1+c_5)\|\nabla f(\zv_t)\|^2,q\|\nabla f(\zv_t)\|)<\infty,
\end{equation}
almost surely. Furthermore, with a careful choice of $c_1$ and $c_4$ such that $\beta^2+c_1+c_4<1$, by Proposition \ref{prop:1}, we conclude that
\begin{equation*}
    \sum_{t=1}^{\infty}\alpha_t\|\vv_t\|^2<\infty.
\end{equation*}
For the next part of the proof, we want to show that the inequality in Lemma \ref{lem:3} holds.
Define the "error" sequence
\[
\wv_t := \gv^{\mathrm{DP}}_t - \nabla f(\zv_t)
\quad\text{and}\quad
\alpha_t' := \tfrac{\alpha_t}{1-\beta}.
\]
By $\zv_{t+1}=\zv_t-\alpha'_t(\nabla f(\zv_t)+\wv_t)$. Since $\nabla f$ is $L$-Lipschitz,
for any $\tau\ge 1$,
\begin{align*}
&\big|\|\nabla f(\zv_{t+\tau})\|-\|\nabla f(\zv_t)\|\big|\\
&\le \big\|\nabla f(\zv_{t+\tau})-\nabla f(\zv_t)\big\|\\
&\le L\Big\|\zv_{t+\tau}-\zv_t\Big\|\\
&\le L\Big\|\sum_{i=t}^{t+\tau-1}\alpha'_i\big(\nabla f(\zv_i)+\wv_i\big)\Big\|\\
&\le L\sum_{i=t}^{t+\tau-1}\alpha'_i\|\nabla f(\zv_i)\|
+ L\Big\|\sum_{i=t}^{t+\tau-1}\alpha'_i \wv_i\Big\|.
\end{align*}
Therefore, setting $b_t:=\|\nabla f(\zv_t)\|$,
\begin{equation}
|b_{t+\tau}-b_t|
\le
L\sum_{i=t}^{t+\tau-1}\alpha'_ib_i
+ L\Big\|\sum_{i=t}^{t+\tau-1}\alpha'_i \wv_i\Big\|.
\label{eq:orabona-lipschitz}
\end{equation}

We first show $\sum_t \alpha'_tb_t^2<\infty$. From equation \eqref{eq:SHB_sum}, the sequence $\sum_t \alpha'_t\Phi_t(\xv_t)$ is finite.
Using $\|\zv_t-\xv_t\|\to 0$, $L$-smoothness implies
\(
\|\nabla f(\zv_t)\|^2
\le 2\|\nabla f(\xv_t)\|^2 + 2L^2\|\zv_t-\xv_t\|^2.
\)
Thus%, by \eqref{eq:summab} and $\sum_t\alpha_t^2<\infty$,
\[
\sum_t \alpha'_t\|\nabla f(\zv_t)\|^2
\le 2\sum_t \alpha'_t\|\nabla f(\xv_t)\|^2
+ \frac{2L^2\beta^2}{(1-\beta)^3}\sum_t \alpha_t\|v_t\|^2
<\infty.
\]
% (The first sum is finite since $\Phi_t(\xv_t)\ge (1-\eta_t)\|\nabla f(\xv_t)\|^2$ and the second by \eqref{eq:summab}.)

Next, we show $\sum \alpha'_t \wv_t$ converges almost surely.
Decompose
\begin{align*}  
\wv_t
&= \underbrace{\big(\mathrm{clip}_q(\nabla f(\xv_t;\xi_t))-\Eb_t[\mathrm{clip}_q(\nabla f(\xv_t;\xi_t))]\big)}_{=:U_t}\\
&+ \underbrace{\Eb_t[\mathrm{clip}_q(\nabla f(\xv_t;\xi_t))]-\nabla f(\zv_t)}_{=:T_t}
+ \underbrace{q\zeta_t}_{=:G_t}.
\end{align*}
Here $U_t$ is a martingale difference with $\Eb_t\|U_t\|^2\le q^2$ and $G_t$ is zero-mean Gaussian with variance $q^2\sigma_{\mathrm{DP}}^2 I$. Hence $\sum_t \alpha'_t U_t$ and $\sum_t \alpha'_t G_t$ converge almost surely due to being martingales bounded in $\Lc^2$ \cite[Theorem 12.1]{williams1991probability}. 

Finally, the transfer term
$\sum_t \alpha'_t T_t$
expands to $\sum_t\alpha'_t\tfrac{qL\beta}{1-\beta}\|\vv_t\|+\max(\|\nabla f(\zv_t)\|-q,0)$ as shown in equation \eqref{eq:cs_v}, and $\sum_t \alpha_t\|\vv_t\|<\infty$. If $q\geq 1$, then $\max(\|\nabla f(\zv_t)\|-q,0)\leq\min(\|\nabla f(\zv_t)\|^2,q\|\nabla f(\zv_t)\|)$. And the convergence of $\sum_t\alpha'_t\min(\|\nabla f(\zv_t)\|^2,q\|\nabla f(\zv_t)\|)$ follows by equation \eqref{eq:conv_last_iterate1}.

With (i) and (ii), Lemma \ref{lem:3} with $p=2$, $b_t=\|\nabla f(\zv_t)\|$, $\alpha_t'=\alpha_t/(1-\beta)$, and \eqref{eq:orabona-lipschitz} implies
\(
\|\nabla f(\zv_t)\| \to 0,
\)
almost surely. For DP-NAG, define $\yv_t=\xv_t+\beta(\xv_t-\xv_{t-1})$ (the look-ahead point) and use the update
$\xv_{t+1}=\yv_t-\alpha_t \gv^{\mathrm{DP}}_t(\yv_t)$.
The same proof applies with $\zv_t$ replaced by $\yv_t$ and $\alpha'_t=\alpha_t/(1-\beta)$.
\end{proof}

\begin{remark}
    Almost sure convergence of $f(\xv_t)$ trivially follows if $f$ is $\mu$-strongly convex as $\mu(f(\xv_t)-f^*)\leq\tfrac{1}{2}\|\nabla f(\xv_t)\|^2$.
\end{remark}

\section{Conclusion}

In this paper, we established the first almost sure convergence guarantees for differentially private stochastic gradient descent (DP-SGD) and its momentum variants, including DP-SHB and DP-NAG. Our analysis adapts supermartingale techniques to handle the combined challenges of gradient clipping and Gaussian noise injection, which break the unbiasedness and smooth descent properties that underlie classical SGD proofs. We showed that, under standard assumptions, the iterates converge almost surely to stationary points in the non-convex setting and to the global minimizer in the strongly convex setting.
Our results provide pathwise convergence, ensuring that individual runs of DP-SGD stabilize rather than merely converging in expectation. This strengthens the theoretical foundation for deploying DP-SGD in practice, where guarantees for single trajectories are often more relevant than averaged behaviors.

Several directions remain open. Our analysis provides almost sure convergence guarantees regardless of the choices of the clipping parameter $q$ or the variance of the injected noise $\sigma_{DP}^2$. However, increasing either of these parameters will naturally slow down the convergence rate in practice. Deriving convergence rates that depend on these parameters will be an interesting area for future work.

\bibliography{main}
\bibliographystyle{abbrv}

\end{document}